\documentclass[12pt,a4paper]{article}

\usepackage[normalem]{ulem}

\usepackage{amssymb}
\usepackage{amsthm}
\usepackage{url}
\usepackage{multirow}
\usepackage{enumitem}
\usepackage{graphicx}
\usepackage{color}

\begin{document}
\newtheorem{lemma}[theorem]{Lemma}



\title{Mutual information and sensitivity analysis for feature selection in customer targeting: a comparative study}


\author{N\'estor Ruben Barraza$^1$ \and
S\'ergio Moro$^2$ \and
Marcelo Ferreyra$^3$ \and
Adolfo de la Pe\~na$^4$}

\date{
\begin{center}
$^1$ Universidad Nacional de Tres de Febrero, Caseros, Argentina\\
$^2$ Instituto Universit\'ario de Lisboa (ISCTE-IUL), ISTAR-IUL, Lisboa, Portugal \\
ALGORITMI Research Centre, University of Minho, Guimarães, Portugal \\
$^3$ Dataxplore, Trenque Lauquen, Argentina\\
$^4$ Boldt Gaming, Buenos Aires, Argentina\\
\end{center}
}

\maketitle

\begin{abstract}

Feature selection is a highly relevant task in a data-driven knowledge discovery project. Several techniques have been developed aiming at finding the features that influence most an outcome to predict, including mutual information and, in recent years, the data-based sensitivity analysis. The present research focus on analyzing the advantages and disadvantages of each of these two techniques, by applying both to a bank telemarketing case. Thereafter, a logistic regression model is built on the tuned set of features identified by each of the two techniques as the most influencing set of features on the success of a telemarketing contact, in a total of 13 features for mutual information and 9 features for the data-based sensitivity analysis. The latter performs better for lower values of false positives while the former is slightly better for a higher false positive ratio. Thus, mutual information becomes a better choice if bank managers intend to reduce slightly the cost of contacts without risking losing a high number of successes. Such results show that mutual information, although not recent, is still a valid method for feature selection. On the other side, the data-based sensitivity analysis’ selection achieved good prediction results with less features.
\\


\end{abstract}


\section{Introduction}

Customer targeting (CT) is a classical problem addressed by Business Intelligence (BI) methods and techniques. It involves finding the right customers who to target at within the context of a marketing campaign for selling the campaign product or service \cite{turbandecision}. Typical cutting edge approaches include using data mining (DM) for unveiling the potential knowledge through patterns of information hidden in big data repositories \cite{witten2005data}. DM adopts the best practices inherited both from classical statistics and artificial intelligence, in an attempt to take advantage from both to enhance knowledge extraction from raw data \cite{grossman2013data}.

In recent years, industries worldwide have experienced peaks and troughs of enthusiasm arisen from the high expectations of benefiting from novel technologies and approaches introduced by DM \cite{gandomi2015beyond}. Discovering the best customers for targeting at a specific moment in time has proven to be NP-hard \cite{nobibon2011optimization}. In real world, a vast number of characteristics and contextual specificities may potentially affect customer's receptivity for acquiring a product. While recent technologies and DM procedures have progressively been increasing their capabilities of analyzing large quantities of data, there is a growing need derived from data availability to identify which are the features that may potentially influence an outcome and which are the ones that are irrelevant and should be discarded, for these may directly contribute for misleading DM algorithms \cite{moro2014data}. Also, the larger the number of features, the slower and more complex is the execution of the DM algorithm in pursuit for the best possible solution, given the exponential growth of possibilities that the algorithm needs to explore \cite{wu2014data}. Hence, feature selection is a highly relevant task in any DM approach, constituting a key step where a large portion of the global effort should be spent \cite{domingos2012few}.

Several techniques have been introduced and applied for feature selection. In \cite{chandrashekar2014survey} the authors conducted a survey, identifying three main methods, filter, wrapper, and embedded, while also mentioning the application of other techniques such as using unsupervised learning and ensemble methods. Table~\ref{table:ftselectioncategories} summarizes their categorization.

\begin{table}[!h]
\centering
\scalebox{0.8}{
\begin{tabular}{|p{3cm}|p{8cm}|p{4cm}|}
\hline
\textbf{Method} & \textbf{Description} & \textbf{Examples of techniques} \\ 
\hline
\multirow{2}{*}{Filter} & \multirow{2}{\linewidth}{variable ranking techniques as the principle criteria for variable selection by ordering} & Correlation criteria\\
& & Mutual information\\
\hline
\multirow{3}{*}{Wrapper} & \multirow{3}{\linewidth}{use the predictor as a black box and the predictor performance as the objective function to evaluate the variable subset} & Sequential selection algorithms\\
& & Heuristic search algorithms\\
& & Sensitivity analysis\\
\hline
Embedded & reduce the computation time taken up for reclassifying different subsets which is done in wrapper methods by incorporating the feature selection as part of the training process & SVM-RFE (Recursive Feature Elimination) \\
\hline
\multirow{3}{*}{Others} & \multirow{3}{\linewidth}{several techniques that do not fit in the remaining three methods} & Clustering \\
& & Ensemble\\
& & Sensitivity analysis\\
\hline
\end{tabular}
}
\caption{Feature selection methods (adapted from \cite{chandrashekar2014survey})}
\label{table:ftselectioncategories}
\end{table}

Mutual information (MI) is one of the most widely adopted feature selection techniques, with the earliest studies dating back to the nineteen nineties \cite{battiti1994using}. The concept underlying MI is to measure the mutual dependence between two random features by identifying how much information of one of the features can be obtained from the other feature. Thus, it is linked to the entropy of a random feature, given by the amount of information held in the feature \cite{paninski2003estimation}.

The usage of sensitivity analysis (SA) for feature selection in DM projects has been studied at least since the dawn of the new millennium 
\cite{embrechts2003data}. The main idea behind SA is to assess model's sensitivity to the variation of each of the input features on the predicted outcome: the more sensitive is the model, the more relevant is the effect of changing the input feature on the outcome. In this context, SA may be considered a wrapper method, according to categorization identified in Table~\ref{table:ftselectioncategories}, even though SA may also be included within the model training process (thus, in this latter case, it would become an embedded method).

Most DM projects need to include a data preparation step, where usually occurs a feature selection procedure \cite{Han}. Being CT a typical problem addressed through DM, makes of it an ideal candidate for the application of feature selection methods. Thus, several studies were published related to the application of feature selection to CT \cite{tan2012feature,moroframework}.
A recent study authored by \cite{tan2013impact} verses on the impact of feature selection in direct marketing. Their work analyzed three filter methods for feature selection (correlation-based feature selection, subset consistency, and symmetrical uncertainty), concluding that symmetrical uncertainty resulted in better models, outperforming others using the remaining two methods studied and a model without any feature selection procedure.

While there are several studies published on feature selection using MI, and a few using SA, none performed a direct comparison on both methods to assess the pros and cons on using each. Furthermore, even though a handful of recent studies were found comparing feature selection methods through practical applications \cite{tan2013impact}, none considered SA.
The main contributions of this paper are as follows:

\begin{itemize}
\renewcommand{\labelitemi}{$\bullet$}
\item Comparing mutual information with sensitivity analysis for feature selection, by testing both methods on a real case CT problem;
\item Assessing the advantages and disadvantages of adopting each of the methods for feature selection, by cross-validating the results achieved on the experiments with real data with the background provided by the literature on the subject;
\item Drawing the insights on each method that may lead scholars and researchers on the adoption of each for a wide range of data-driven approaches to address real-world problems.
\end{itemize} 

This paper is organized as follows. Section \ref{it} presents a summary on the literature for MI, SA and feature selection applied to CT. In Section \ref{arc}, the materials and methods adopted for the experiments are described. Results and evaluation of applying both methods are discussed in Section \ref{sec:res}. Finally, conclusions are drawn in the last section.

\section{Background}\label{it}

\subsection{Mutual information}\label{sec:mi}

Entropy and mutual information (MI) are well known concepts in Communications and Information Theory. They were originally introduced by Claude Shannon in a seminal paper \cite{shannon48}, in order to find the optimal coding of a source on one hand and a noisy channel on the other. Entropy is related to uncertainty  or information content of a random variable. From this point of view, an event $i$ having probability of occurrence $p_i$ has an information content of:

\begin{equation}
\label{informationcontent}
I = -\log p_i
\end{equation}
 
The base of logarithm defines the unit, base two logarithm gives units in bits. When more probable is the event, less information gives its occurrence. The expectation of (\ref{informationcontent}) gives the average information content of such set of events:

\begin{equation}
\label{entropy}
H(X) = - \sum p_i \log p_i 
\end{equation}
The expression (\ref{entropy}) is the entropy of the random variable $X$ in such a way that the event $i$ corresponds to the value $x_i$, i.e. $p(x_i) = p_i$.

Entropy is bounded by the cardinality of the set of outcomes: $ H(X) \le \log |\mathcal{X}|$ and attains its maximum when the events follow a uniform distribution $ p_i = \frac{1}{|\mathcal{X}|}$. The bigger the entropy, the more random are the events, then, the occurrence of an event gives more information, though they are less predictable. 
 
Considering two random variables with a given joint probability $p(X,Y)$, the joint entropy is defined as:

\begin{equation}
\label{jointentropy}
H(X,Y) = - \sum_{x \in \mathcal{X}, y \in \mathcal{Y}} p(x,y) \log p(x,y) 
\end{equation}

And the conditional entropy is defined as:

\begin{equation}
\label{condentropy}
H(X|Y) = - \sum_{x \in \mathcal{X}, y \in \mathcal{Y}} p(x,y) \log p(x|y) 
\end{equation}

Mutual information between two random variables is defined as follows:

\begin{equation}
\label{mutualinformation}
I(X;Y) = \sum_{x \in \mathcal{X}, y \in \mathcal{Y}} p(x,y) \log \frac{p(x,y)}{p(x) \; p(y)} 
\end{equation}

From (\ref{mutualinformation}) we can derive a relation between mutual information and entropy:

\begin{eqnarray}
\label{mutualinformationentropy}
I(X;Y) &=& H(X) - H(X|Y) \nonumber \\
&=& H(Y) - H(Y|X) 
\end{eqnarray}

Mutual information definition can be extended to sets of random variables $X^n = \{X_1, X_2, \cdots, X_n\}$ and $Y^n = \{Y_1, Y_2, \cdots, Y_n\}$:

\begin{equation}
\label{mutualinformationsets}
I(X^n;Y^n) = \sum_{x^n \in \mathcal{X}^n, y^n \in \mathcal{Y}^n} p(x^n,y^n) \log \frac{p(x^n,y^n)}{p(x^n) \; p(y^n)} 
\end{equation}
where $\mathcal{X}^m$ and $\mathcal{Y}^m$ are the set of outcomes of $x^n$ and $y^m$.

A situation where redundant information can be removed arises when variables are connected in a Markov chain, as shown in Figure \ref{conn2}.  A well known relation for this case is given by the data processing inequality $I(X_1;X_2) \ge I(X_1;Y)$, an alternative inequality  is demonstrated in Lemma \ref{lema2} in a similar way.

\begin{figure}[!ht]
  \centering
  \includegraphics[width=2.5in,page=2,scale=0.6]{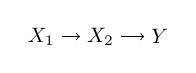}
  \caption{Connection of variables for lemma \ref{lema2}}\label{conn2}
\end{figure}

\begin{lemma}
\label{lema2}
If the random variables $X_1$, $X_2$, $Y$ are connected in a Markov chain $ X_1 \rightarrow X_2 \rightarrow Y$, then:
\begin{enumerate}
\item\label{prim} $I(X_1,X_2;Y) = I(X_2;Y)$ 
\item\label{sec} $I(X_2;Y) \ge I(X_1;Y)$
\end{enumerate}
\end{lemma}

\begin{proof}
Applying twice the chain rule to the mutual information:
\begin{eqnarray}
I(X_1,X_2;Y) &=& I(X_1;Y|X_2) + I(X_2;Y) \label{cr1}\\
&=& I(X_2;Y|X_1) + I(X_1;Y) \label{cr2}
\end{eqnarray}

By the Markov property:

\begin{equation}
\label{mp}
I(X_1;Y|X_2) = 0
\end{equation}

Replacing (\ref{mp}) in (\ref{cr1}) we prove the first part. 

Taking into account that the mutual information is always greater than 0 (see for example \cite{Cover}), from  (\ref{cr1}), (\ref{cr2}) and (\ref{mp}) we get $ I(X_2;Y) \ge I(X_1;Y)$, and the lemma is proved. 

\end{proof}

Then, we can point out that  these functions from the information theory are quite suitable to eliminate redundant information, not generally taken into account with other methods.


A communication channel is a device or medium capable of transmitting information. The input information is carried out to the output. Since any mechanism of transmitting information is not perfect, some noise is introduced in the process.  Then, the input and output information are not the same but related. We can use the mutual information as a measure of that relation. We will consider thereafter that $X$ is the random variable at the input and $Y$ is the output random variable.

According to the source coding theorem, we know that the entropy is a measure of the average bits of information necessary to code the outcomes of a given random variable. In this way, $H(X)$ is a measure of the input information to the channel, $H(Y)$ the information content at the output, $I(X;Y)$ the transmitted information, and taking into account the relations (\ref{mutualinformationentropy}): $H(Y|X)$ is a measure of the noise introduced by the channel. The conditional entropy $H(X|Y)$ is called equivocation or ambiguity and it must be subtracted to the input information in order to obtain the transmitted information. According to the channel coding theorem, the Mutual Information gives the channel capacity and determines the maximum rate of information transmitted by the channel, see details in \cite{shannon48} and \cite{Cover}.\\

Mutual Information based feature selection consists in choosing the set of variables raising the most of the information of the output variable following a given criterion. The process starts by adding features from those carrying the most of the information until a stopping criterion is reached. Since the mutual information between the output variable $Y$ and a given subset of input variables $X^m$ is given by:
\begin{equation}
\label{miset}
I(X^m;Y) = H(Y) - H(Y|X^m)
\end{equation}
the estimation of $I(X^m;Y)$ involves the estimation of $P(Y|X^m)$. Then, the bigger the set $X^m$, the less reliable is the estimation of $P(Y|X^m)$. Therefore, the maximum of the product of information gain with reliability of estimation is used as the stopping rule, as indicated in the flow chart shown in fig. \ref{mimethod}. Other stopping rules for variables addition were considered in \cite{Hall:2003:BAS:951848.951906}. Note that the relative information gain $IG^m=\frac{I(X^m;Y)}{H(Y)}$ goes from $0$ when $X^m$ and $Y$ are independent and $H(Y|X^m) = H(Y)$ to 100\% when $Y$ is a deterministic function of $X^m$ and $H(Y|X^m) = 0$.

\begin{figure}[!ht]
  \centering
  \includegraphics[width=3in,page=1,scale=0.6]{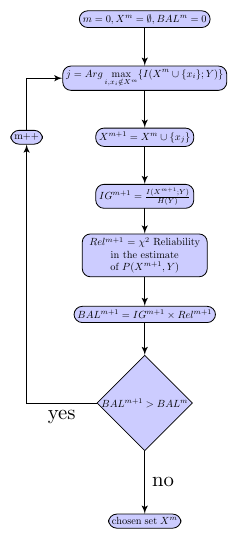}
  \caption{MI feature selection method}\label{mimethod}
\end{figure}

As a result, MI theory may help to cope with the task of feature selection by identifying the subset of features that minimizes $H(Y|X)$, i.e., that maximizes the information underlying in the dataset, $I(X;Y)$. Thus, features introducing further entropy in the communication channel may be discarded, helping to guide DM algorithms in building a model that understands the intrinsic relationships underneath the original data.

\subsection{Sensitivity analysis}\label{sec:sa}

The usage of sensitivity analysis (SA) for providing insights on complex models dates back from the nineteen nineties, with a special emphasis on climate and environmental models \cite{homma1996importance}. However, the advent of DM has introduced highly complex models with intrinsic convoluted relations that are hardly disentangled. Some of the most widely used of those models include machine learning algorithms such as neural networks in various formats and versions, and support vector machines. SA has been proposed and analyzed in the literature for input feature evaluation from data mining models. In Ref. \cite{kewley2000data} a computationally efficient one-dimensional method is presented, by varying one input at a time through its possible range of values and keeping the remaining input features constant. Subsequent studies explored further SA as a means for understanding the impact each of the features that contributed to a model implementation had on the predicted outcome \cite{embrechts2003data,kondapaneni2007visualization}.
By providing a procedure to evaluate the relevance of input features from models, several studies have included SA as a method for feature selection, hence choosing the most relevant features and discarding the least relevant \cite{liu2012feature}. Although SA requires that a model is previously available for assessing feature relevance since it focus solely on the features, it can virtually be applied to any type of predictive model. Recent developments resulted in novel SA techniques such as the data-based SA 
(DSA), introduced by \cite{cortez2013using} in 2013. This procedure uses random samples from the data used to train the model for assessing feature relevance by changing the input features simultaneously, thus considering the relations between features (Figure~\ref{fig:dsa}).

\begin{figure}[!htp]
  \centering
  \includegraphics[scale=0.7]{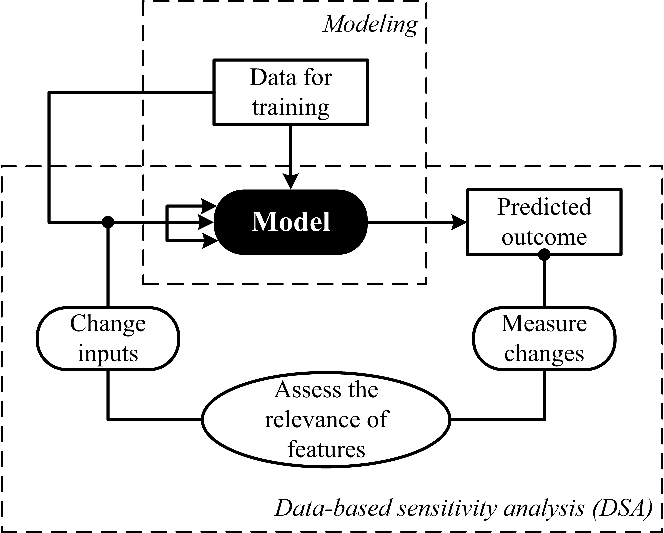}
  \caption{Data-based sensitivity analysis}
  \label{fig:dsa}
\end{figure}

DSA has been applied since then in a large spectrum of domains and problems, such as bank marketing \cite{moro2015using}, wine quality assessment \cite{cortez2009modeling}, jet grouting formulations \cite{tinoco2012application} and social media performance metrics \cite{moro2016predicting}. However, the only study using DSA specifically for feature selection is the work by Moro, Cortez and Rita, which applied such technique in a bank telemarketing case and resulted in a thread of published articles \cite{moro2014data,moro2015using,moroframework}. Their work adopted DSA but lacked in assessing the advantages and disadvantages of sensitivity analysis when compared to other feature selection methods. The present paper represents the first attempt in filling such gap.

\subsection{Feature selection in customer targeting}\label{sec:ftct}

Customer targeting (CT) is the marketing procedure of optimizing the selection of customers who to target within the context of a marketing campaign to meet campaign goals, usually, the acquisition of a product or service \cite{cole2012internet}. CT can be viewed as a branch of an integrated Customer Relationship Management (CRM) strategy with a focus on building customer equity \cite{richards2008customer}. Other terms that are directly related to CT include direct marketing and database marketing, with the former being almost a synonymous \cite{kim2004intelligent}, while the latter can be also associated with the need for a customer database to support CRM strategies \cite{richards2008customer}. It should be noted that the task of discovering the most likely positive responders to customer targeting has been proven to be NP-hard \cite{nobibon2011optimization}. CT provides an interesting ground for testing predictive machine learning techniques, with a large number of published studies alleging the discovery of predictive knowledge that may be used to benefit the success of CT \cite{kim2004intelligent,ngai2009application}. Nevertheless, few of those works have seen a real production environment, effectively leveraging business \cite{pinheiro2014heuristics}.

Feature selection is a key task in every DM projects \cite{domingos2012few}. The main goal is to find the minimum set of features that optimize results translated in terms of model accuracy in fitting new data for the problem being addressed \cite{witten2005data}. Also, by reducing the number of features used for modeling, the procedure for training the model becomes lesser computationally expensive, making it feasible to be executed on a daily or more frequent basis, for incorporating the subtle changes derived from immediate previous contacts \cite{liu2005toward}. For example, a bad news on the company or product widely spread through social media may directly affect the subsequent contacts \cite{goldenberg2007npv}.
Thus, model retraining for learning with new occurrences needs to occur often. One option is to use a rolling windows procedure where the window of data for training the model slides for keeping pace with time, an approach that may be adopted for several time evolving problems such as stock markets \cite{romero2015nonlinearities} and telemarketing \cite{moro2014data}. The longer the algorithm takes to run the modeling procedure, the more likely the model does not adapt quickly enough to new information. Hence, selecting the right amount of features is in demand for problems with constant shifts in the influence features have on the outcome.

Feature selection using MI has been a subject of research in numerous problems. Moreover, studies are usually devoted to testing new feature selection approaches to well-known datasets, not focusing explicitly on the advantages to the business associated with the problems being addressed \cite{hoque2014mifs,hancer2015multi}. However, no studies were found on feature selection using MI specifically focusing on CT, only a few papers published related to customer churning \cite{idris2012customer,verbraken2014profit}. DSA application to CT for feature selection has been the subject of study of Moro et al., as stated in Section~\ref{sec:sa}. The present study is focused in filling such void while at the same time performing a novel comparison between both methods, MI and DSA.

\section{Materials and Methods}\label{arc}

\subsection{Real case}

Bank telemarketing is a specific case of direct marketing where the customers of a bank are contacted and offered products or services through phone calls, although other direct channels such as email may be used \cite{moro2014data}. For the experiments presented in this paper, the dataset published in the University of California Machine Learning Repository (http://archive.ics.uci.edu/
ml/) was adopted. Such dataset was studied by numerous scholars and researchers, as the high number of page hits shows, above two hundred thousand. As a result, several studies have been published using its data, with the most for assessing machine learning and DM algorithms' capabilities \cite{verbraken2014profit}, and a few for feature selection \cite{vajiramedhin2014feature}.
This dataset encompasses a total of 41,188 phone contacts conducted by human agents from a Portuguese bank between 2008 and 2010, with the goal of selling an attractive long-term deposit, in an attempt of retaining customers' financial assets in the institution. It should be stressed that all contacts are real, implying that it represents a real problem and to which feature selection may provide interesting benefits in reducing the features needed for modeling the outcome, choosing only influencing features while at the same time reducing model retraining duration.
Each contact is characterized by twenty features, with some related to personal customer data (e.g., age), others to the contact itself (e.g., call duration) and previous calls made within the context of older campaigns (e.g., the outcome of previous contact), and the remaining related to the social and economic context that characterizes the country (e.g., number of employed people).
Table~\ref{table:listft} describes the list of features. More details can be obtained from Ref. \cite{moro2014data}. The target outcome is the 21\textsuperscript{st} feature from the dataset, concealing a binary value (yes/no) which represents the contact result: ``yes'' if the customer subscribed the deposit; ``no'' otherwise.

\begin{table}[!h]
\centering
\scalebox{0.8}{
\begin{tabular}{|p{2.5cm}|p{10cm}|p{2.5cm}|}
\hline
\textbf{Feature} & \textbf{Type and description} & \textbf{Group} \\ 
\hline
age & numeric & \multirow{7}{*}{Customer}\\
job & type of job (categorical - 12 possible values)&\\
marital & marital status (categorical - 4 possible values)&\\
education & (categorical - 8 possible values)&\\
default & has credit in default? (categorical: ``no'',``yes'',``unknown'')&\\
housing & has housing loan? (categorical: ``no'',``yes'',``unknown'')&\\
loan & has personal loan? (categorical: ``no'',``yes'',``unknown'')&\\
\hline
contact & contact communication type (categorical: ``cellular'',``telephone'') & \multirow{4}{*}{Contact}\\
month & last contact month of year (categorical)&\\
day\_of\_week & last contact day of the week (categorical)&\\
duration & last contact duration, in seconds (numeric)&\\
\hline
campaign & number of contacts performed during this campaign & \multirow{4}{*}{Other}\\
pdays & number of days that passed by after the client was last contacted from a previous campaign &\\
previous & number of contacts performed before this campaign and for this client &\\
poutcome & outcome of the previous marketing campaign &\\
\hline
emp.var.rate & employment variation rate - quarterly indicator (numeric) & \multirow{5}{*}{Context}\\
cons.price.idx & consumer price index - monthly indicator (numeric)&\\
cons.conf.idx & consumer confidence index - monthly indicator (numeric)&\\
euribor3m & euribor 3 month rate - daily indicator (numeric)&\\
nr.employed & number of employees - quarterly indicator (numeric)&\\
\hline
\end{tabular}
}
\caption{List of input features (from https://archive.ics.uci.edu/ml/datasets/Bank+Marketing)}
\label{table:listft}
\end{table}

\subsection{Experimental procedure}

The bank telemarketing dataset was first assessed in terms of feature relevance by both methods studied, MI and DSA. Each method has its own specificities and procedures, as mentioned in sections \ref{sec:mi} and \ref{sec:sa}. MI evaluates the amount of information concealed in each of the features, whereas DSA assesses the model in terms of the influence on the outcome by changing the input features. 
The experimental setup for the case of MI is solely the dataset with the data, as detailed in Section~\ref{sec:mi}, while the DSA required that a model was previously built for assessing feature relevance in terms of the sensitivity of the model to changes on input features, as shown in Figure~\ref{fig:dsa}.
Therefore, while MI selects the features according to the information each of them contains when compared to the remaining, DSA ranks features in terms of relevance for the model. For the latter, all the features that did not encompass individually at least 2\% of relevance were discarded.
Figure~\ref{fig:proc} summarizes the approach followed.

\begin{figure}[h]
  \centering
  \includegraphics[scale=0.7]{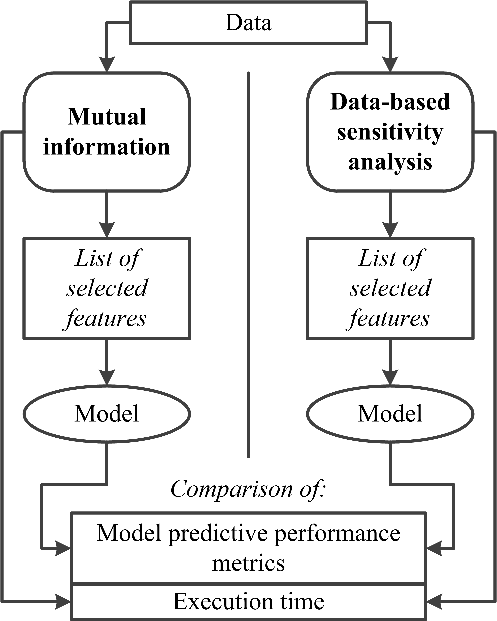}
  \caption{Procedure}
  \label{fig:proc}
\end{figure}
In order to simulate a sliding window, the dataset was divided successively in different training and testing sets using a ten fold cross-validation procedure. The testing set is composed by the contacts whose willingness of acquiring the product constitutes the outcome to predict \cite{refaeilzadeh2009cross}. This set is chosen by a window that takes 10\% of the total records, starts at the first record and shifts to the next 10\% of records without overlapping. In each experiment, the training set is composed by the other 90\% records. At each fold, from the training set, the features were selected and used for building a predictive model which is applied to the testing set, in a procedure similar to Ref. \cite{moro2014data}. At the end of the 10 fold experiments, a score of the probability of acquiring the offered product was computed for each contact. That score was then used to build the ROC curve and find the confusion matrices shown below for different cutoff probabilities.
The selected features are then used for implementing a simple logistic regression model to fit data for predicting the outcome on the contacts. The usage of logistic regression provides a direct means for measuring how modeling with the selected feature behaves in both cases, for allowing a direct comparison. While more complex machine learning techniques could be used (e.g., neural networks or support vector machines), the goal of the present study is to facilitate a comparison of both feature selection procedures, not putting emphasis on the modeling scenarios, where other studies have already focused on the analyzed dataset. Also, to keep coherence in all experiments, the logistic regression was also chosen for extracting feature relevance during feature selection from DSA. Additionally, since DSA is based on ranking feature relevance, for computing the model, all top ranked features with a summed relevance of at least 90\% were included, discarding the remaining.
Finally, the prediction results are analyzed in the light of comparing both methods using both the receiver operating characteristic curve (ROC) and confusion matrices. Also, computational performance is evaluated, for a lighter method in terms of execution may allow a global frequently run learning procedure to be scheduled more often, as stated in Section~\ref{sec:ftct}.

\section{Results and Discussion}\label{sec:res}

As remarked in \cite{moro2014data}, a predictive model does not hold any knowledge on future occurrences of the problem; thus, the dataset should be striped off of any feature only known after contact execution, such as the call duration. Furthermore, in a real predictive system, a campaign is launched without knowing when will the calls be made, as these depend on both agent and especially client availability. Therefore, for the experimental setup, all features related to the current campaign were removed, namely: "contact", ''month'', ''day'' and ''duration''. The only exception was "campaign", which deserved a detailed analysis: this field indicates the number of contacts performed during this campaign, i.e., how many times was the contact rescheduled, which may happen due to several reasons, such as a machine answered the call and the agent decided to reschedule it, or the client asked to be recalled later. By taking into account that a telemarketing campaign can last a year, one may consider a dynamical model that take this field into account, incorporating multiple calls to the same client within the same campaign; therefore, this feature was included in the present analysis.

Since the software powerhouse\footnote{http://www.dataxplore.com.ar/tecnologia.php\#Powerhouse} performs a segmentation process based on information theory, such product was chosen for obtaining the metrics. The software calculates the mutual information between each attribute and the output variable separately by estimating the joint probability, then, it chooses the set of variables carrying most of the information according to the criterion explained in sec. \ref{sec:mi}. The set of the selected variables are shown in table \ref{table:t1}.

\begin{table}
\centering
\scalebox{0.8}{
\begin{tabular}{|c|c|c|}
\hline
Name & Gain $=\frac{I(X;Y)}{H(Y)}$ (Average) & NF \\ 
\hline
cons.conf.idx  & 17.02\% & 9 \\
euribor3m      & 8.83\%  & 7 \\
campaign       & 5.80\%  & 7 \\
housing        & 5.80\%  & 10 \\
emp.var.rate   & 5.80\%  & 2 \\
marital        & 5.50\%  & 10 \\
nr.employed    & 5.07\%  & 3 \\
age            & 4.98\%  & 10 \\
education      & 4.37\%  & 10 \\
loan           & 4.27\%  & 6 \\
poutcome       & 2.81\%  & 8 \\
pdays          & 2.09\%  & 6 \\
job            & 1.74\%  & 10 \\
\hline
\end{tabular}
}
\caption{Selected variables by the MI method, NF stands for the number of folds the attribute has been taken into account because of its relevance.}
\label{table:t1}
\end{table} 

Elimination of redundant information occurs in the present data with the variable ''conf.price.idx''. This variable can be totally predicted by ''cons.conf.idx'' and ''nr.employed'' as shown in table \ref{table:t2}.

\begin{table}
\centering
\scalebox{0.8}{
\begin{tabular}{|c|c|c|}
\hline
Name & Gain $\frac{I(X;Y)}{H(Y)}$ & Confidence \\ 
\hline
cons.conf.idx & 94.47\%	& 96.11\% \\
nr.employed & 100.00\% & 95.72\% \\
\hline
\end{tabular}
}
\caption{Information content and prediction of the variable ''conf.price.idx''.}
\label{table:t2}
\end{table} 

As a result, we can say that the variables mentioned before are connected as shown in Fig. \ref{conn3}, in a similar way as it was presented in lemma \ref{lema2}.

\begin{figure}[!ht]
  \centering
  \includegraphics[width=5in,page=12]{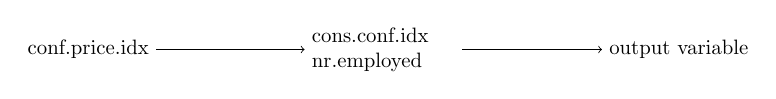}
  \caption{Prediction of variable ''conf.price.idx''}\label{conn3}
\end{figure}


Similar information content tables are shown for the other discarded variables ''default'' and ''previous''.

\begin{table}
\centering
\scalebox{0.8}{
\begin{tabular}{|c|c|c|}
\hline
Name & Gain $\frac{I(X;Y)}{H(Y)}$ & Confidence \\ 
\hline
poutcome & 83.11\% & 97.44\% \\
nr.employed & 86.02\% & 95.62\% \\
pdays & 87.59\% & 95.33\% \\
job & 89.42\% & 95.05\% \\
education & 91.94\% & 94.92\% \\
\hline
\end{tabular}
}
\caption{Information content and prediction of the variable ''previous''.}
\label{table:t22}
\end{table} 

\begin{table}
\centering
\scalebox{0.8}{
\begin{tabular}{|c|c|c|}
\hline
Name & Gain $\frac{I(X;Y)}{H(Y)}$ & Confidence \\ 
\hline
euribor3m & 6.73\%	& 96.11\% \\
education & 11.61\% &95.22\% \\
age & 15.87\% & 95.05\% \\
job & 23.39\% & 93.28\% \\
campaign & 35.62\% & 79.30\% \\
marital & 45.07\% & 69.70\% \\
cons.conf.idx & 55.51\% & 58.51\% \\
\hline
\end{tabular}
}
\caption{Information content and prediction of the variable ''default''.}
\label{table:t23}
\end{table} 

For the case of DSA feature selection, the R statistical tool was chosen (https://cran.r-project.org/). R is an open source framework focusing on data analysis problems, allowing contributions with independent packages developed by numerous researchers worldwide \cite{ihaka1996r}. Additionally, the rminer package implements the DSA algorithm \cite{cortez2010data}; therefore, it was adopted for all experiments related to DSA. Selected variables by DSA method are shown in table \ref{table:t12}. As previously explained, it should be noted that within each fold of execution, all features summing up to 90\% of global relevance were included. Interestingly, these features were considered for each of the ten folds, emphasizing its relevance for the problem and data being addressed.

\begin{table}[!h]
\centering
\scalebox{0.8}{
\begin{tabular}{|c|c|c|c|}
\hline
Name & Relative relevance (average) & Summed relevance & NF \\ 
\hline
emp.var.rate & 26.29\% & 26.81\% & 10\\
pdays & 13.60\% & 40.41\% & 10\\
nr.employed & 12.36\% & 52.77\% & 10\\
cons.price.idx & 10.86\% & 63.64\% & 10\\
default & 10.28\% & 73.92\% & 10\\
poutcome & 6.50\% & 80.42\% & 10\\
job & 3.53\% & 83.95\% & 10\\
cons.conf.idx & 3.43\% & 87.38\% & 10\\
campaign & 3.22\% & 90.61\% & 10\\
marital & 3.03\% & 93.64\% & 9\\
age & 2.28\% & 95.92\% & 8\\
education & 1.66\% & 97.58\% & 0\\
euribor3m & 0.89\% & 98.47\% & 0\\
previous & 0.73\% & 99.20\% & 0\\
loan & 0.43\% & 99.63\% & 0\\
housing & 0.37\% & 100.00\% & 0\\
\hline
\end{tabular}
}
\caption{Selected variables by the DSA method.}
\label{table:t12}
\end{table}

Once a set of highly relevant features are selected based on the information content for MI, or based on model sensitivity to such features for DSA (Table~\ref{table:ftselected}), a predictive model may be built using any available DM technique, from naive Bayes to the most recent support vector machines, neural networks and genetic algorithms (e.g., \cite{moro2014data}). Considering the focus of this study is feature selection, a simple logistic regression method was chosen for predicting contact outcome, for assessing the efficiency and accuracy of the features selected with both methods. This is the simplest algorithm of those analyzed by \cite{moro2014data}. The ROC curves for both models are drawn in Figure~\ref{roc-curve}. For further understanding the effects of using each predictive model built on each set of features, four confusion matrices are computed. On Tables \ref{table:t3MI} and \ref{table:t3DSA}, results are extracted considering a typical cutoff probability 0.5, i.e., in which the most likely outcome is considered a success if the model predicts it with 50\% or more of probability, whereas on  Tables \ref{table:t31MI} and \ref{table:t31DSA} the cutoff is lowered to just 10\%, to account for the fact that this particular bank intends to increase efficiency with a especial emphasis on avoiding loosing successful contacts, considering lost deposit subscriptions directly implicates on missing business opportunities for retaining important financial assets in a crisis period (thus, the cost of loosing a successful contact is much higher than the gain of avoiding an unuseful unsuccessful contact) \cite{moro2014data}.

\begin{table}[!h]
\centering
\scalebox{0.8}{
\begin{tabular}{|c|c|c|}
\hline
Feature & MI & DSA \\ 
\hline
emp.var.rate & X & X \\
pdays & X & X \\
nr.employed & X & X \\
cons.price.idx &  & X \\
default &  & X \\
poutcome & X & X \\
job & X & X \\
cons.conf.idx & X & X \\
campaign & X & X \\
marital & X &  \\
age & X & \\
education & X &  \\
euribor3m & X &  \\
previous &  &  \\
loan & X &  \\
housing & X &  \\
\hline
\end{tabular}
}
\caption{Selected features for both methods.}
\label{table:ftselected}
\end{table}

\begin{figure}[!ht]
    \centering
  \includegraphics[width=3in,page=11]{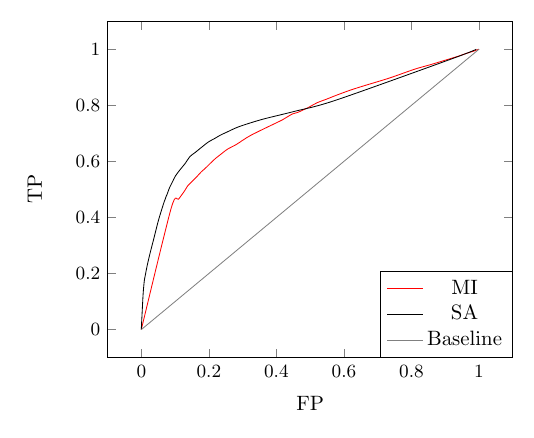}
 \caption{ROC curves.}
\label{roc-curve}
\end{figure}



\begin{table}
\parbox{.45\linewidth}{
\begin{tabular}{c|cc}
 & \multicolumn{2}{c}{Predicted as} \\
Target & failure & success \\
\hline
failure & 28583 & 7964 \\
success & 1817 & 2824 \\
\hline
\end{tabular}
\caption{Confusion Matrix for the MI method for cutoff probability 50\%.}
\label{table:t3MI}
}
\hfill
\parbox{.45\linewidth}{
\begin{tabular}{c|cc}
 & \multicolumn{2}{c}{Predicted as} \\
Target & failure & success \\
\hline
failure & 36110 & 438 \\
success & 3746 & 894 \\
\hline
\end{tabular}
\caption{Confusion Matrix for the DSA method for cutoff probability 50\%.}
\label{table:t3DSA}
}
\end{table}

\begin{table}
\parbox{.45\linewidth}{
\begin{tabular}{c|cc}
 & \multicolumn{2}{c}{Predicted as} \\
Target & failure & success \\
\hline
failure & 445 & 36102 \\
success & 25 & 4616 \\
\hline
\end{tabular}
\caption{Confusion Matrix for the MI method for cutoff probability 10\%.}
\label{table:t31MI}
}
\hfill
\parbox{.45\linewidth}{
\begin{tabular}{c|cc}
 & \multicolumn{2}{c}{Predicted as} \\
Target & failure & success \\
\hline
failure & 27969 & 8579 \\
success & 1431 & 3209 \\
\hline
\end{tabular}
\caption{Confusion Matrix for the DSA method for cutoff probability 10\%.}
\label{table:t31DSA}
}
\label{table:t31}
\end{table} 


It is interesting to observe from the confusion matrices that an increasing number of false positives (FP) turns MI slightly better in predicting contact outcome than DSA, whereas DSA is clearly better for smaller FPs. Nevertheless, as stated previously, the results of MI maybe preferable if accounted that the cost of making a call is far less than the benefits of hitting a customer willing of getting the product.
Another advantage of MI method is processing time, as results are obtained faster, since no heavy CPU consume is involved for modeling as it is the case of DSA. For the experiments conducted, MI procedure took just a few seconds, whereas DSA took around half a minute in an Intel\textsuperscript{TM} I3 processor. This is a highly relevant benefit if more complex model techniques are used, as DSA depends on a model being built first. Furthermore, such advantage may be particularly emphasized for larger datasets.



Comparing tables \ref{table:t12} and \ref{table:t1} several important remarks may be stated:

\begin{enumerate}
\item Both feature selection methods clearly achieve different results.
\item The most relevant features for one method is considered as little relevant for the other. Such is the case of ''emp.var.rate'' and ''cons.conf.idx''.
\item Contrarily to MI, the number of features selected by DSA remains the same along the 10-fold experiments.
\item The set of features selected by MI is bigger than that selected by DSA. 
\end{enumerate}

The previous remarks lead to an interesting analysis. It seems to be that a given estimate like DSA does not reveal the inherent dependency among strongly correlated features. This situation can occur if the evaluation estimate considers a value where the independent variable does not lead to big variations in the estimate, like in a possible flat portion in the regression curve. This may be observed from Tables \ref{table:t12} and \ref{table:t1}, where the features selected by DSA remain the same through the folds, contrarily to the MI method. Also, from the MI method, a feature can be discarded if there is not enough values as to get a good level of confidence, explaining why many variables were discarded in some folds, as shown in Table~\ref{table:t1}, column NF.

It is possible to observe that MI has not taken into account the variables ''conf.price.idx'', ''default'' and ''previous''; such finding has risen the interest in analyzing how each of these variables are related to the remaining selected by MI. This analysis can be made from Tables \ref{table:t2},  \ref{table:t22} and \ref{table:t23}. Table \ref{table:t2} shows that ''conf.price.idx'' can be totally predicted by ´´cons.conf.idx´´ and ´´nr.employed´´. 
A similar case occurs with ´´previous´´, this variable can also be almost totally predicted by other selected variables, as it can be seen from 
Table \ref{table:t22}. Also, Table \ref{table:t23} shows that most of the information carried out by ´´default´´ is carried by other of the selected variables, thus it can be discarded.

Another interesting result comes out from the confusion matrices shown in Tables \ref{table:t3MI} and \ref{table:t31MI}. Similar behavior results for the two cutoff probability of success, it can be seen that the MI method gives confusion matrices with very good specificity and bad sensitivity. Thus, the model hits many customers willing to get the product while at the same time failing by contacting many clients that would reject the offer. Hence, for the empirical experiments conducted, DSA may be qualified as as more conservative, while the MI method is preferred when the cost of making a call is low and the income of selling the product is high.

The ROC curves are similar, although the one obtained from the MI method is slightly worse given its greater number of false positives, as shown in Figure~\ref{roc-curve}.

\section{Conclusions}\label{sec:conc}
In this study, a comparison was conducted between two renowned feature selection methods, mutual information (MI) and the data-based sensitivity analysis (DSA). Advantages and disadvantages of both methods were shown. This important information can be used to decide the best method to use in a particular application. For the empirical procedure, a dataset containing more than forty thousand of problem instances of the case of bank telemarketing was chosen. In this experiment, the advantages of applying the information theory concepts in order to eliminate redundant attributes were translated in a small subset of highly relevant features which enabled modeling faster and more accurately the outcome of clients subscribing or not a deposit. Also, the method allows getting the information content easily and rapidly, which allows that to be applied to big data sets. Since variables carrying most of the information of the output variable are selected by the proposed method, results have shown that a simple prediction algorithm such as logistic regression can be performed with good modeling results. 
On the other side, the data-based sensitivity analysis has the disadvantage of requiring a model for extracting feature relevance. Such drawback can halt a data mining project if the initial dataset holds a high number of features. Nevertheless, DSA does not require to dive deeply into the model for understanding which features are lesser relevant, for it is based on assessing outcome variation by also changing input features through their range of possible values. Using the tuned set of features obtained from each methods, in a total of 13 features for MI and 9 features for DSA from the initial 20, it is possible to observe from a logistic regression model built on each of these two sets that the receiver operating characteristic curve from DSA outperforms MI model in the lower values of false positives, while MI is slightly better for a higher false positive ratio. Thus, if the goal of marketing managers is to reduce the number of calls made at the cost of eventually loosing some successful contacts (true positives), then DSA feature selection resulted better for this case; otherwise, MI's feature selection took a small lead. Such conclusion is highlighted in the confusion matrices obtained, with MI's model achieving better results for predicting successes, while DSA outperforms MI on predicting failed contacts, i.e., when the client refused the deposit offered. For this specific case, loosing a successful contact implicates eventual loss of the client's financial asset, thus it is preferable to achieve a higher accuracy on predicting successes at the expense of wasting additional calls on unfruitful contacts.
Nevertheless, the results are conclusive in that MI, although a rather old method, still achieves results comparable to other more recent methods, such as DSA.






\section*{Acknowledgements}
One of us (NRB) would like to thank Universidad Nacional de Tres de Febrero for financial support under grant no. 32/15 201.


\bibliography{paper}

\section*{The authors}

\subsection*{Nestor R. Barraza}
received his Msc. and Ph.D. in Electronic Engineering from the University of Buenos Aires
in 1993 and 1998. He is currently a full Professor in the National University of Tres de Febrero, Argentina. Hi is also with the School of Engineering of the University of Buenos Aires. He has authored and co-authored many scientific and academic papers and participated in several local and international conferences. His research interests include Information Theory and Coding, Communications, Business Intelligence and Software Reliability. 

\subsection*{S\'ergio Moro}

is an Assistant Professor at Instituto Universitário de Lisboa (ISCTE-IUL), and member of ISTAR-IUL and ALGORITMI Research Center. He holds a PhD in Information Sciences and Technologies and an MSc in Management Information Systems, both from ISCTE-IUL, and a 5 year BSc in Computer Engineering from Instituto Superior Técnico (University of Lisbon). His research appears in journals such as Decision Support Systems, Expert Systems with Applications, and Journal of Business Research. He has worked for 15 years (2001-2016) at Montepio Bank, as a Software Engineer, Project Manager and Business Intelligence \& Analytics Manager.

\subsection*{Marcelo Ferreyra}
is the founder of powerhouse, a Data Mining software tool. He is currently CEO at dataxplore and independent consultant of Business Intelligence and Data Mining. 

\subsection*{Adolfo de la Pe\~na}
received his Msc. in Electronic Engineering from the University of Buenos Aires
in 1990 and his Msc. in Business Administration from the University of Salvador in 2013. He is currently Engineering Manager at the Computer Department at Boldt Gaming S.A. His research interest includes Business Intelligence and Data Mining.

\end{document}